\documentclass[twocolumn,10pt]{asme2e}

\usepackage{subcaption}
\usepackage{epsfig} 
\usepackage{bm}
\usepackage{amssymb}
\usepackage{amsmath}
\newcommand{\comment}[1]{}
\newtheorem{proposition}{Proposition}
 	
\newtheorem{definition}{Definition} 
\confshortname{IDETC/MSNDC 2018}
\conffullname{the ASME 2018 International Design Engineering Technical Conferences \&\\
              International Conference on Multibody Systems, Nonlinear Dynamics, and Control}

\confdate{August 26-29, 2018 }
\confyear{2018}
\confcity{Quebec City}
\confcountry{Canada}

\papernum{DETC2018-85429}

\title{Rigid Body Dynamic Simulation with Multiple Convex Contact Patches}

\author{Jiayin Xie
    \affiliation{
	Department of Mechanical Engineering\\
	Stony Brook University\\
	Stony Brook, New York 11794\\
    Email: jiayin.xie@stonybrook.edu
    }	
}

\author{Nilanjan Chakraborty \thanks{Address all correspondence to this author.}
    \affiliation
	Department of Mechanical Engineering\\
	Stony Brook University\\
	Stony Brook, New York 11794\\
    Email: nilanjan.chakraborty@stonybrook.edu
    }
\begin{document}

\maketitle    


\begin{abstract}
{\it We present a principled method for dynamic simulation of rigid bodies in intermittent contact with each other where the contact is assumed to be a non-convex contact patch that can be modeled as a union of convex patches. The prevalent assumption in simulating rigid bodies undergoing intermittent contact with each other is that the contact is a point contact. In recent work, we introduced an approach to simulate contacting rigid bodies with convex contact patches (line and surface contact). In this paper, for non-convex contact patches modeled as a union of convex patches, we formulate a discrete-time mixed complementarity problem where we solve the contact detection and integration of the equations of motion simultaneously. Thus, our method is a geometrically-implicit method and we prove that in our formulation, there is no artificial penetration between the contacting rigid bodies. We solve for the equivalent contact point (ECP) and contact impulse of each contact patch simultaneously along with the state, i.e., configuration and velocity of the objects. We provide empirical evidence to show that if the number of contact patches between two objects is less than or equal to three, the state evolution of the bodies is unique, although the contact impulses and ECP may not be unique. We also present simulation results showing that our method can seamlessly capture transition between different contact modes like non-convex patch to point (or line contact) and vice-versa during simulation. }
\end{abstract}

\comment{
\begin{nomenclature}
\entry{A}{You may include nomenclature here.}
\entry{$\alpha$}{There are two arguments for each entry of the nomemclature environment, the symbol and the definition.}
\end{nomenclature}

The spacing between abstract and the text heading is two line spaces.  The primary text heading is  boldface in all capitals, flushed left with the left margin.  The spacing between the  text and the heading is also two line spaces.
}
\begin{figure}
\centering
\includegraphics[width=2.5in]{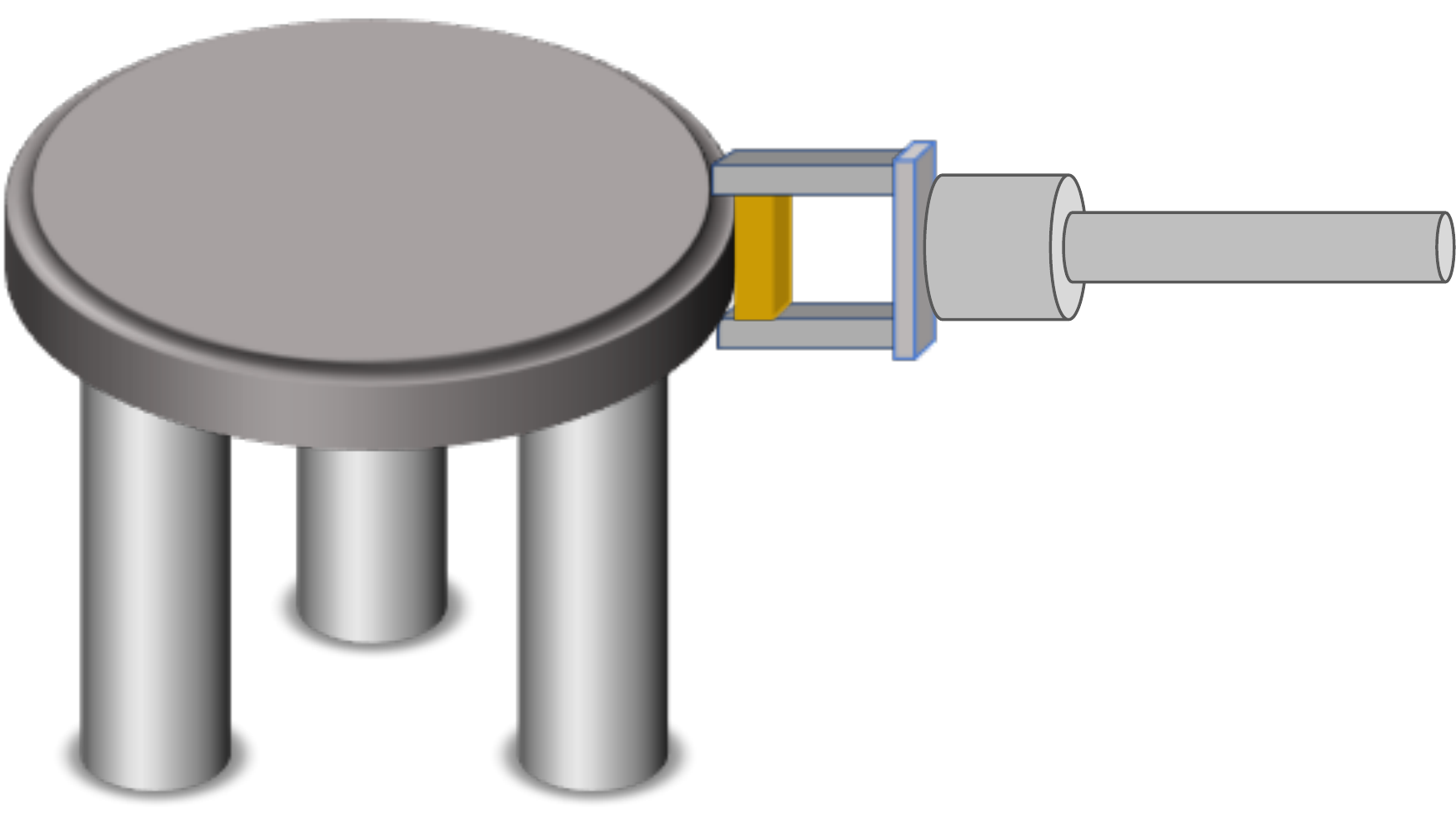}
\caption{A robot pushing a stool with three legs.}

\label{figure:push_stool} 
\end{figure} 
\section*{INTRODUCTION}

Rigid body dynamic simulation is a key enabling technology in solving robotic manipulation~\cite{Kolbert2017,MaD11} and mechanical design problems~\cite{SongTVP04}. 
Robotic manipulation such as prehensile pushing~\cite{Kolbert2017} and in-hand manipulation~\cite{MaD11} involves point and surface contacts between a gripper and a rigid body. Furthermore, the occurrence of multiple intermittent contacts makes the prediction of the motion more complicated. There are applications in which the contact between two objects may be over topologically disconnected regions. For example, Figure~\ref{figure:push_stool} shows a robot with a manipulator pushing a three-legged stool, where the contact between the ground and the stool is a union of three disks. Such situations may arise when a mobile robot with a manipulator is navigating a room and wants to push the movable obstacle (stool) out of its way. State-of-the-art dynamic simulation algorithms usually assume point contact between two objects (except~\cite{Xie2016}), which is clearly violated in Figure~\ref{figure:push_stool}, and there are no well-principled approaches to solve such problems. In this paper, we seek to develop principled algorithms for simulating rigid bodies in intermittent contact where the contacts can be modeled as union of multiple patch contacts.  

\begin{figure}
\centering
\includegraphics[width=3.5in]{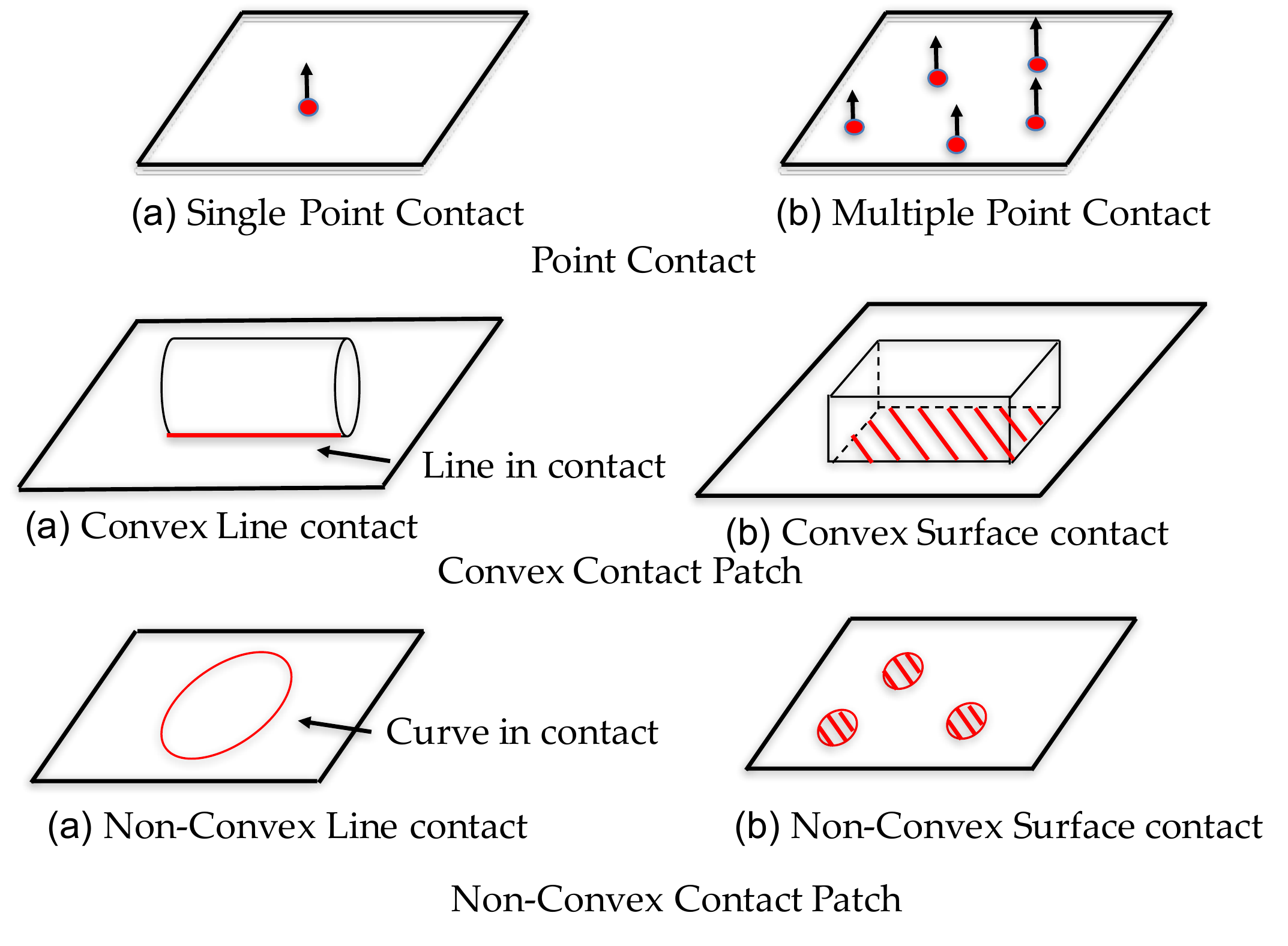}
\caption{Different types of contact between one object with a flat surface. Our focus in this paper is on simulating rigid bodies with type of contact shown in last row, figure (b).}
\label{figure:different_contact_case} 
\end{figure} 

Figure~\ref{figure:different_contact_case} shows the key types of contact between objects. Existing mathematical models for motion of objects with intermittent contact like Differential Algebraic Equation (DAE) models~\cite{Haug1986} and Differential Complementarity Problem (DCP) models~\cite{Cottle2009,Trinkle1997,PfeifferG08} assume the contact between the two objects is a single point contact (top left in Figure~\ref{figure:different_contact_case}). However, for convex contact patch (middle row in Figure~\ref{figure:different_contact_case}), the point contact assumption is not valid. In such case, multiple contacts point are usually chosen in an ad hoc manner, which can lead to inaccuracies in simulation. Recently, we developed an approach~\cite{Xie2016} to simulate contacting rigid bodies with convex contact patches (line and surface contact). In this paper, we focus on the non-convex surface contact problem where the non-convex contact patch that can be modeled as a union of convex patches with flat surface (bottom right of Figure~\ref{figure:different_contact_case} ). Such situations arise when a robot is manipulating objects placed on a horizontal plane.

For single convex contact patch, we know that there exists a unique point on the contact surface where the integral of total moment due to normal force acting on this point is zero. This point is used to model line or surface contact as a point contact and thus it is called the {\em equivalent contact point} (ECP)~\cite{Xie2016}. Using the concept of ECP, in~\cite{Xie2016}, we present a principled method for simulating intermittent contact with convex contact patches (line and surface contact). This method solves for the ECP as well as the contact impulses by incorporating the collision detection within the dynamic simulation time step. This method is called the {\em geometrically implicit time-stepping method} because the geometric information of contact points and contact normal are solved as a part of the numerical integration procedure. In this paper, we extend the method in~\cite{Xie2016} to model union of convex contact patches between two objects in intermittent contact. We use an ECP to model the effect of each contact patch and solve for the ECP and its associated contact wrenches on each contact patch separately. The ECP and contact wrenches are computed simultaneously along with the state of the objects by augmenting the equations of motion of the objects with the contact constraints of non-penetration. We prove that even though we are modeling each contact patch with an equivalent contact point, the contact constraints are always satisfied at the end of the time-step and there is no artificial penetration between the objects. Through simulation studies, we present empirical evidence that for less than or equal to three contact patches, although the contact wrenches and ECP may not be unique, the state of the object is unique (this is different from a single convex contact patch where the ECP and contact wrenches are unique). Furthermore, for pure translation, we prove that the state of the object at the end of time step as well as contact impulses can be computed analytically. We also present simulation results showing that our method allows seamless transition between multiple patch contacts to point or line contacts and vice-versa.

\section*{RELATED WORK}

We model the continuous time dynamics of rigid bodies that are in intermittent contact with each other as a Differential Complementarity problem (DCP).  
Let $\bm{u}\in \mathbb{R}^{n_1}$,  $\bm{v}\in \mathbb{R}^{n_2}$ and let $\bm{g}$ :$ \mathbb{R}^{n_1}\times \mathbb{R}^{n_2} \rightarrow \mathbb{R}^{n_1} $, $\bm{f}$ : $ \mathbb{R}^{n_1}\times \mathbb{R}^{n_2} \rightarrow \mathbb{R}^{n_2}$ be two vector functions and the notation $0 \le \bm{x} \perp \bm{y} \ge 0$ imply that $\bm{x}$ is orthogonal to $\bm{y}$ and each component of the vectors is non-negative. 
\begin{definition}
\cite{Facchinei2007} The differential (or dynamic) complementarity problem is to find $\bm{u}$ and $\bm{v}$ satisfying
$$\dot{\bm{u}} = \bm{g}(\bm{u},\bm{v}), \ \ \ 0\le \bm{v} \perp \bm{f}(\bm{u},\bm{v}) \ge 0 $$
\end{definition}
\begin{definition}
The mixed complementarity problem is to find $u$ and $v$ satisfying
$$\bm{g}(\bm{u},\bm{v})=0, \ \ \ 0\le \bm{v} \perp \bm{f}(\bm{u},\bm{v}) \ge 0.$$
If the functions $\bm{f}$ and $\bm{g}$ are linear, the problem is called a mixed linear complementarity problem (MLCP), otherwise, the problem is called a mixed nonlinear complementarity problem (MNCP). Our continuous time dynamics model is a DCP whereas our discrete-time dynamics model is a MNCP. 
\end{definition}

The DCP model formulates the intermittent contact between bodies in motion as a complementarity constraint~\cite{Lotstedt82, AnitescuCP96, Pang1996, StewartT96, Liu2005, Drumwright2012, Todorov2014}. 
DCP models are solved numerically with time-stepping schemes. The time-stepping problem is: {\em given the state of the system and applied forces, compute an approximation of the system one time step into the future. } Solving this problem repeatedly will give an approximate solution to the equations of motion of the system. There are different assumptions for forming the discrete equations of motion, which makes the system Mixed Linear Complementarity problem (MLCP)~\cite{AnitescuP97, AnitescuP02} or mixed non-linear complementarity problem (MNCP)~\cite{Tzitzouris01,NilanjanChakraborty2007}. The MLCP problem linearizes the friction cone constraints and the distance function between two bodies (which is a nonlinear function of the configuration), sacrificing accuracy for speed. Depending on whether the distance function is approximated, the time-stepping schemes can also be divided into geometrically explicit schemes~\cite{AnitescuCP96, StewartT96} and geometrically implicit schemes~\cite{Tzitzouris01}. In geometrically explicit schemes, at the current state, a collision detection routine is called to determine separation or penetration distances between the  bodies, but this information is not incorporated as a function of the unknown future state at the end of the current time step. A goal of a typical time-stepping scheme is to guarantee consistency of the dynamic equations and all model constraints at the end of each time step. However, since the geometric information is obtained and approximated only at the start of the current time-step,
then the solution will be in error. Thus, in~\cite{Xie2016,NilanjanChakraborty2007}, we use a geometrically implicit time stepping scheme for solving convex contact patches problem, which is also the method used in this paper. The resulting discrete time problem is a MNCP.

\section*{DYNAMIC MODEL}
\label{sec:dynmod}
We now present the geometrically implicit optimization-based time-stepping scheme for modeling the dynamic simulation with multiple intermittent unilateral contacts between two objects.
Note that a contact between two objects is a union of multiple convex contact patches.
The dynamic model includes (a) Newton-Euler equations (b) kinematic map relating the generalized velocities to the linear and angular velocities (c) friction law for each contact patch (d) contact constraints incorporating the geometry of the contact patches. 

We will introduce the notations and write the equations of motion for a single object in contact with another object. The vector describing the position of the center of mass and the orientation of the object is $\bm{q}$  ($\bm{q}$ can be $6 \times 1$ or $7\times 1$ vector depending on the representation of the orientation). For numerical simulation we use unit quaternion to represent the orientation.
Let $\bm{\nu}$ be the generalized velocity concatenating the linear ($\bm{v}$) and spatial angular ($^s\bm{\omega}$) velocities. The total number of contact patches is $n_c$. For each contact patch $i$, let
$\lambda_{n_i}$ ($p_{n_i}$) be the magnitude of normal contact force (impulse),
$\lambda_{t_i}$ ($p_{t_i}$) and $\lambda_{o_i}$ ($p_{o_i}$) be the orthogonal components of the friction force (impulse) on the tangential plane, and
$\lambda_{r_i}$ ($p_{r_i}$) be the frictional force (impulse) moment about the contact normal.

\section*{ Newton-Euler Equations}
The  Newton-Euler equations are as follows:
\begin{equation} 
\begin{aligned}
\label{eq1}
\bm{M}(\bm{q})
{\dot{\bm{\nu}}} &= 
\sum_{i = 1}^{n_c}\bm{W}_{n_i}\lambda_{n_i}+
\sum_{i = 1}^{n_c}\bm{W}_{t_i} \lambda_{t_i}+
\sum_{i = 1}^{n_c}\bm{W}_{o_i} \lambda_{o_i}
\\&+\sum_{i = 1}^{n_c}\bm{W}_{r_i}\lambda_{r_i}+
\bm{\lambda}_{app}+\bm{\lambda}_{vp}
\end{aligned}
\end{equation}
\noindent
where $\bm{M}(\bm{q}) = \left[ \begin{matrix} &m \bm{I}_3 \ &0\\ &0 \ &{^s\mathcal{I}}_{cm}\end{matrix}\right]$ is a symmetric, positive definite $6 \times 6$ matrix, which contains mass matrix $m \bm{I}_3$ ($\bm{I}_3$ is a $3 \times 3$ identity matrix) and  inertia matrix ${^s\mathcal{I}}_{cm} = \bm{R} \mathcal{I}_{cm} \bm{R}^T$. Here $\bm{R}$ is the $3 \times 3$ rotation matrix from body frame to world frame and $\mathcal{I}_{cm}$ is the inertia matrix in the body frame. $\bm{\lambda}_{app}$ is the  $6 \times 1$ vector of external forces (including gravity) and moments, $\bm{\lambda}_{vp}$ is the $6 \times 1$ vector of Coriolis and centripetal forces, $\sum_{i = 1}^{n_c}\bm{W}_{n_i}\lambda_{n_i}$, $\sum_{i = 1}^{n_c}\bm{W}_{t_i}\lambda_{t_i}$, $\sum_{i = 1}^{n_c}\bm{W}_{o_i}\lambda_{o_i}$ and $
\sum_{i = 1}^{n_c}\bm {W}_{r_i}\lambda_{r_i}$ are the sum of wrenches of the normal contact forces, frictional contact forces, and frictional moments on each contact patch. And $n_c$ is the total number of contact. 
Let $(\bm{n}_i,\bm{t}_i,\bm{o}_i)$ be unit vectors of the contact frame and $\bm{r}_i$ be the vector from center of gravity to the ECP of $i$th contact patch, expressed in the world frame.
\begin{equation}
\begin{aligned}
\label{equation:wrenches}
\bm{W}_{n_i} =  \left [ \begin{matrix} 
\bm{n}_i\\
\bm{r}_i\times \bm{n}_i
\end{matrix}\right]
\quad 
\bm{W}_{t_i} =  \left [ \begin{matrix} 
\bm{t}_i\\
\bm{r}_i\times \bm{t}_i
\end{matrix}\right]
\\
\bm{W}_{o_i} =  \left [ \begin{matrix} 
\bm{o}_i\\
\bm{r}_i\times \bm{o}_i
\end{matrix}\right]
\quad 
\bm{W}_{r_i} =  \left [ \begin{matrix} 
\bm{0}\\
\ \ \bm{n}_i \ \
\end{matrix}\right]
\end{aligned}
\end{equation}

To discretize Equation~\eqref{eq1}, we use a backward Euler time-stepping scheme. Let $t^u$ denote the current time and $h$ be the time step, the superscript $u$ represents the beginning of the current time and the superscript $u+1$ represents the end of the current time.
Let $\dot{\bm{\nu}} \approx ( {\bm{\nu}}^{u+1} -{\bm{\nu}}^{u} )/h$, Equation~\eqref{eq1} becomes:
\begin{equation}
\begin{aligned}
\label{equ:discrete_NE}
\bm{M}^{u} {\bm{\nu}}^{u+1} &= 
\bm{M}^{u}{\bm{\nu}}^{u}+\sum_{i = 1}^{n_c}\bm{W}_{n_i}^{u+1}p^{u+1}_{n_i}+\sum_{i = 1}^{n_c} \bm{W}_{t_i}^{u+1}p^{u+1}_{t_i} \\&
+\sum_{i = 1}^{n_c}\bm{W}_{o_i}^{u+1}p^{u+1}_{o_i}+\sum_{i = 1}^{n_c}\bm{W}_{r_i}^{u+1}p^{u+1}_{r_i}+\bm{p}^{u}_{app}+\bm{p}^{u}_{vp}\\
\end{aligned}
\end{equation}
where all the forces in Equation~\eqref{eq1} becomes impulses.

\subsection*{Kinematic Map}
The kinematic map is given by $\bm{\dot{q}} = \bm{G}(\bm{q}) \bm{\nu}$
where $\bm{G}$ is the matrix mapping the generalized velocity of the body to the time derivative of the position and the orientation.
To discretizee the above equation, let $\dot{\bm{q}} \approx( {\bm{q}}^{u+1} -{\bm{q}}^{u} )/h$. Therefore, 
\begin{equation}
\label{equ:kinematic}
\bm{q}^{u+1} =\bm{q}^{u}+h\bm{G}(\bm{q}^u)\bm{\nu}^{u+1}
\end{equation}

\subsection*{Friction Model for each contact patch}
We use a friction model for each contact patch that is based on the maximum power dissipation principle and generalizes Coulomb's friction law. It is given by 
\begin{equation}
\begin{aligned}
&{\rm max} \quad -(v_{t_i} p_{t_i} + v_{o_i}p_{o_i} + v_{r_i} p_{r_i})\\
{\rm s.t.} \quad \left(\frac{p_{t_i}}{e_{t_i}}\right)^2 &+\left(\frac{p_{o_i}}{e_{o_i}}\right)^2+\left(\frac{p_{r_i}}{e_{r_i}}\right)^2 - \mu_i^2 p_{n_i}^2 \le 0, \quad i = 1,...,n_c.
\end{aligned}
\end{equation}
\noindent
where $v_{t_i}$ and $v_{o_i}$ are the tangential components of the relative velocity at the ECP of contact patch $i$, $v_{r_i}$ is the relative angular velocity about the normal at the contact patch $i$.
Let $e_{t_i},e_{o_i}$ and $e_{r_i}$ be the given positive constants defining the friction ellipsoid for contact patch $i$ and let $\mu_i$ represents the coefficient of friction at the patch $i$ \cite{Howe1996, Trinkle1997}. This constraint is the elliptic dry friction condition suggested in \cite{Howe1996} based upon evidence from a series of contact experiments. This model states that among all the possible contact forces and moments that lie within the friction ellipsoid, the forces and moment that maximize the power dissipation at the contact (due to friction) are selected.
 
This argmax formulation of the friction law has a useful alternative formulation~\cite{trinkle2001dynamic}
 \begin{equation}
\begin{aligned}
\label{equation:friction}
0&=
e^{2}_{t_i}\mu_i p_{n_i} 
\bm{W}^{T}_{t_i}\cdot
\bm{\nu}^{u+1}+
p_{t_i}\sigma_i\\
0&=
e^{2}_{o_i}\mu_i p_{n_i}  
\bm{W}^{T}_{o_i}\cdot
\bm{\nu}^{u+1}+p_{o_i}\sigma_i\\
0&=
e^{2}_{_ir}\mu_i p_{n_i}\bm{W}^{T}_{r_i}\cdot
\bm{\nu}^{u+1}+p_{r_i}\sigma_i\\
\end{aligned}
\end{equation}
\begin{equation}
\label{equation:friction_c}
0 \le \mu_i^2p_{n_i}^2- p_{t_i}^2/e^{2}_{t_i}- p_{o_i}^2/e^{2}_{o_i}- p_{r_i}^2/e^{2}_{r_i} \perp \sigma_i \ge 0
\end{equation}
where $\sigma_i$ is the magnitude of the slip velocity on contact patch $i$.

\subsection*{Non-penetration constraint for each contact patch}
In complementarity-based formulation of dynamics, the contact constraint for each potential contact is written as 
\begin{equation} \label{equation:normal contact}
0\le \lambda_{n_i} \perp \psi_{n_i}(
\bm{q},t) \ge 0
\end{equation}
where $i = 1,...,n_c$. $\lambda_{n_i}$ is the magnitude of normal contact force at $i$th contact. Here, $\psi_{n_i}(
\bm{q},t)$ is the gap function for $i$th contact with the property $\psi_{n_i}(\bm{q},t) > 0$ for separation, $\psi_{n_i}(\bm{q},t) = 0$ for touching and $\psi_{n_i}(\bm{q},t) < 0$ for inter-penetration. Since $\psi_{n_i}(\bm{q},t)$ usually has no closed form expression, and the contact constraints should be satisfied at the end of the time step, state-of-the-art time steppers~\cite{NME:NME1049, Liu2005, Stewart2000} do the following: (a) use a collision detection algorithm to get the closest point at the beginning of the time-step (b) approximate the distance function at the end of the time step using a first order Taylor's series expansion. Thus, the time-steppers are explicit in the geometric information and the collision detection step is decoupled from the dynamics solution step, where the state of the system and the contact wrenches are computed. In~\cite{NilanjanChakraborty2007}, the authors discussed the limitations of such an approach in terms of undesired inter-penetration between rigid objects, and introduced a method whereby the geometry of the bodies are included in the equations of motion, so that simulation with no artificial inter-penetration can be guaranteed. 

Previous models including~\cite{NilanjanChakraborty2007}, assume point contact. In~\cite{Xie2016}, the authors develop a principled method to model single convex contact patch (Figure (a) and (b) in second row of Figure~\ref{figure:different_contact_case}). They use geometrically implicit time-stepping method from~\cite{NilanjanChakraborty2007} to solve for the equivalent contact point (ECP) on the contact surface, its associated wrench and configurations of the object simultaneously, thus, making the problem well-posed.
In this paper, we extend the method presented in~\cite{Xie2016} to model contact problem with union of convex contact patches. We use ECP to model each contact patch separately, and solve them with their associated contact wrenches and configurations of the objects simultaneously. In the subsequent sections, we provide empirical evidence to show that if the number of contact patches is less than or equal to three, although the ECP and its associated wrenches at each patch are not unique, the state of the objects is unique. 

The guarantee of non-penetration is valid for single point contact between two objects. We need to prove that the guarantee of non-penetration is valid for multiple contact patches. In the next section, we discuss the geometrically implicit method in detail and prove that this method guarantees non-penetration for each convex contact patch and therefore there will be no inter-penetration between two objects with multiple contact patches. 

\section*{CONTACT CONSTRAINTS}
\begin{figure}
\centering
\includegraphics[width=3in]{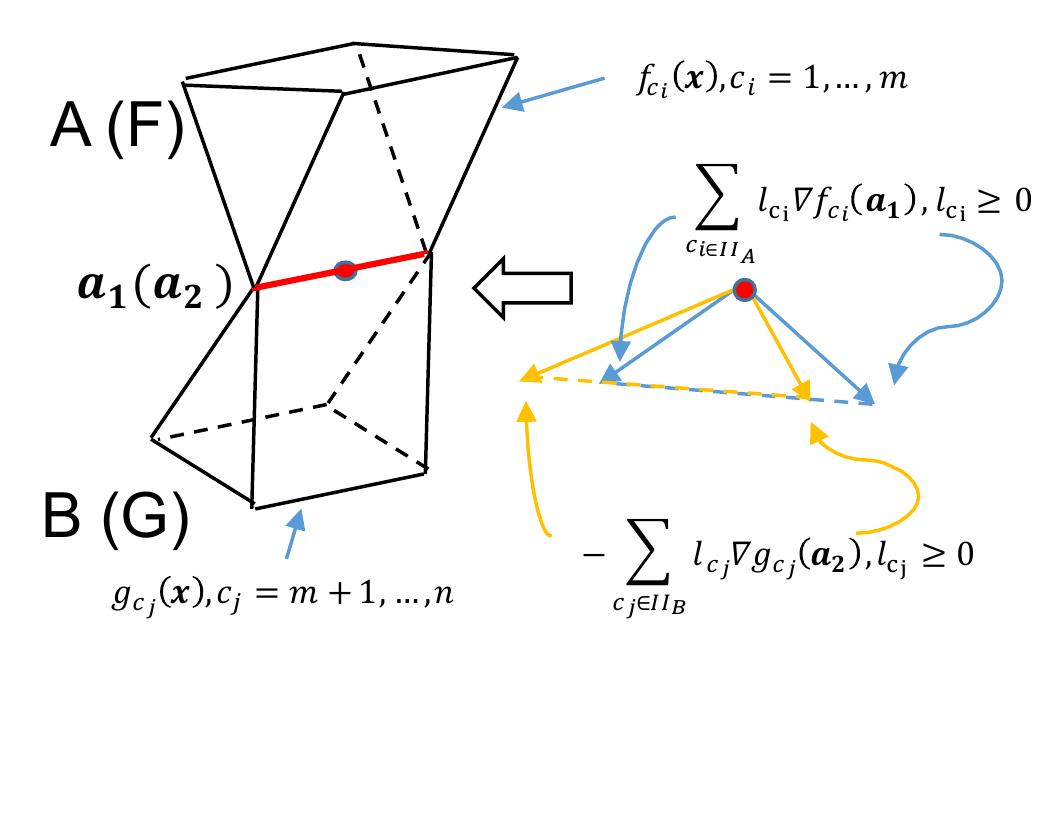}
\caption{Line contact between two convex bodies from object F and G respectively.}
\label{figure:contact_nonpenetration} 
\end{figure} 

We consider two objects F and G, that are modeled by the union of convex bodies. Thus, one or multiple pairs of convex bodies from F and G  can potentially have contact. When a pair of bodies have contact, there is a convex contact patch between them, that can be point, line or surface contact. Therefore, the non-convex contact patches between F and G can be modeled as union of convex patches.

\subsection*{Contact constraints for a single contact patch}
Let us consider a single convex contact patch. As Figure~\ref{figure:contact_nonpenetration} shows, convex body A (one body composing object F) and B (one body composing object G) can be described by intersection of convex inequalities $f_{c_i}(\bm{x}) \le 0, c_i = 1,...,m$ and $g_{c_j}(\bm{x}) \le 0, c_j=m+1,...,n$ respectively. We define $\bm{a}_1$ as the closest point (or ECP) on convex body A  and $\bm{a}_2$ on body B. Because the normal on  $\bm{a}_1$ or $\bm{a}_2$ may not be uniquely defined, so we use normal cones $\mathcal{C}(\bm{A},\bm{a}_1) = \sum_{c_i \in II_A}l_{c_i}\nabla f_{c_i}(\bm{a}_1)$ and $\mathcal{C}(\bm{B},\bm{a}_2) = \sum_{c_j \in II_B}l_{c_j}\nabla g_{c_j}(\bm{a}_2)$ to represent any vector that lies within the cone. The explanation of normal cone is presented in detail at appendix A. 

Since the closest point is outside the body if it is outside at least one of the intersecting surfaces forming the body, the contact complementarity Equation~\eqref{equation:normal contact} can be written as~\cite{NilanjanChakraborty2007}:
\begin{equation}
\begin{aligned}
\label{equation:contact_multiple_comp}
0 \le \lambda_{n} \perp \mathop{max}_{c_i=1,...,m} f_{c_i}(\bm{a}_2) \ge 0\\
0 \le \lambda_{n} \perp \mathop{max}_{c_j=m+1,...,n}g_{c_j}(\bm{a}_1) \ge 0
\end{aligned}
\end{equation}
The solution of the closest points (or ECPs) $\bm{a}_1$ and $\bm{a}_2$ is given by the following minimization problem~\cite{NilanjanChakraborty2007}:
\begin{equation}
\label{equation:optimazation}
(\bm{a}_1,\bm{a}_2) = arg \min_{\bm{\zeta}_1,\bm{\zeta}_2}\{ \|\bm{\zeta}_1-\bm{\zeta}_2 \| \ f_{c_i}(\bm{\zeta}_1) \le 0,\ g_{c_j}(\bm{\zeta}_2) \le 0 \}
\end{equation}
where $c_i=1,...,m$ and $c_j=m+1,...,n$.

Using a slight modification of the KKT conditions for the optimization problem in Equation~\eqref{equation:optimazation}, the closest points (or ECP) should satisfy the following equations:
\begin{align}
\label{equation:re_contact_multiple_1}
\bm{a}_1-\bm{a}_2 = -l_{k_1}(\nabla f_{k_1}(\bm{a}_1)+\sum_{c_i = 1,c_i\neq k_1}^m l_{c_i}\nabla f_{c_i}(\bm{a}_1))\\
\label{equation:re_contact_multiple_2}
\nabla f_{k_1}(\bm{a}_1)+\sum_{c_i = 1,c_i\neq k_1}^m l_{c_i}\nabla f_{c_i}(\bm{a}_1)= -\sum_{c_j = m+1}^n l_{c_j} \nabla g_{c_j} (\bm{a}_2)\\
\label{equation:re_contact_multiple_3}
0 \le l_{c_i} \perp -f_{c_i}(\bm{a}_1) \ge 0 \quad c_i = 1,..,m\\
\label{equation:re_contact_multiple_4}
0 \le l_{c_j} \perp -g_{c_j}(\bm{a}_2) \ge 0 \quad c_j = m+1,...,n
\end{align}
Where $k_1$ represents the index of any one of the active constraints (i.e., the surface on which the closest point lies). We will also need an additional complementarity constraint (any one of the two equations in~\eqref{equation:contact_multiple_comp}) to prevent penetration:
\begin{align}
\label{equation:re_contact_multiple_5}
0 \le \lambda_{n} \perp \mathop{max}_{c_i=1,...,m} f_{c_i}(\bm{a}_2) \ge 0
\end{align}

Note that Equations~\eqref{equation:re_contact_multiple_1} to \eqref{equation:re_contact_multiple_4} are not exactly the KKT conditions of the optimization problem in Equation~\eqref{equation:optimazation} but can be derived from the KKT conditions. This derivation is presented in detail in~\cite{NilanjanChakraborty2007} and is therefore omitted here.

In the proof below, we use separating hyperplane theorem (the detail is presented in appendix B) which states that: two convex non-empty objects can have a common supporting hyperplane at a point which lie on the common region if and only if their interiors are disjoint. The common region is where two objects' boundaries touch or intersect. If objects touch without intersection, common region represents contact patch. 

When distance between bodies A and B is zero, the supporting hyperplane is defined by the normal cone of 
the point lying on the common region. The common region can be point, line segment or surface and we can define one equivalent normal cone for the surface or line segment (see appendix A). Thus the common supporting hyperplane can be defined as the intersection of normal cone $\mathcal{C}(\bm{A},\bm{a}_1)$ and $-\mathcal{C}(\bm{B},\bm{a}_2)$, where $\bm{a}_1, \bm{a}_2$ can be any point that lie on the common region. If $\mathcal{C}(\bm{A},\bm{a}_1) \cap -\mathcal{C}(\bm{B},\bm{a}_2) \neq \emptyset$, there exists common supporting hyperplane, and A and B will not intersect with each other.
\begin{proposition}{When using Equations~\eqref{equation:re_contact_multiple_1} $\sim$~\eqref{equation:re_contact_multiple_5} to model one single convex contact patch, we get the solution for ECPs as the closest points on the boundary of their associated bodies when bodies are separate from each other and we get only touching solution when distance between bodies is zero.  }
\end{proposition}

\begin{proof}
First, when two bodies are separate, $\bm{a}_1 \neq \bm{a}_2$, and Equations~\eqref{equation:re_contact_multiple_1} $\sim$~\eqref{equation:re_contact_multiple_5} will give us the solution for ECPs $\bm{a}_1$ and $\bm{a}_2$ as the closet points on the boundary of bodies. The proof is same as in ~\cite{NilanjanChakraborty2007}. 

For the case when distance between two bodies are not separate, the two bodies either touch each other without penetration or they intersect with each other. To show, the KKT conditions ~\eqref{equation:re_contact_multiple_1} to \eqref{equation:re_contact_multiple_4} will give us the optimal solution for minimization problem (Equation~\eqref{equation:optimazation}), i.e., $\bm{a}_1 = \bm{a}_2$. Furthermore, Equations~\eqref{equation:re_contact_multiple_1} to~\eqref{equation:re_contact_multiple_4} and non-penetration constraint~\eqref{equation:re_contact_multiple_5} together will give us $\bm{a}_1$ and $\bm{a}_2$ as the touching solution, i.e.,:
\begin{enumerate}
\item $\bm{a}_1$ and $\bm{a}_2$ lie on the boundary of body A and B respectively.
\item $\mathcal{C}(\bm{A},\bm{a}_1) \cap -\mathcal{C}(\bm{B},\bm{a}_2) \neq \emptyset$ . 
\end{enumerate}
For (1), let us prove it by contradiction. If $\bm{a}_1$ lies within the interior of the body A, from Equation~\eqref{equation:re_contact_multiple_3}, $f_{c_i}(\bm{a}_1) < 0,\  l_{c_i} = 0 \ \forall c_i = 1,...,m$. From Equation~\eqref{equation:re_contact_multiple_1}, $\bm{a}_1 = \bm{a}_2$, thus $f_{c_i}(\bm{a}_2) < 0 \ \forall c_i = 1,...,m$, which contradicts to Equation~\eqref{equation:re_contact_multiple_5}. Thus $\bm{a}_1$ has to lie on the boundary of body A. If $\bm{a}_2$ lies within the body B, from Equation~\eqref{equation:re_contact_multiple_4}, $g_{c_j}(\bm{a}_2) < 0$, $l_{c_j} = 0  \ \forall c_j = m+1,...,n$. Thus, $\sum_{c_j = m+1}^n l_{c_j} \nabla g_{c_j} (\bm{a}_2) = 0$. Because the left hand side of Equation~\eqref{equation:re_contact_multiple_2} is nonzero, which also leads to a contradiction. Thus $\bm{a}_2$ lies on the boundary of body B. 



Now we need to prove (2). Since $\bm{a}_1$ lies on the boundary of A, there exists a normal cone $\mathcal{C}(\bm{A},\bm{a}_1) \neq \emptyset$. For $\bm{a}_2$, there exists a normal cone $\mathcal{C}(\bm{B},\bm{a}_2)  \neq \emptyset$. The left hand side of Equation~\eqref{equation:re_contact_multiple_2} represents the normal cone $\mathcal{C}(\bm{A},\bm{a}_1) $ and right hand side of this equation represents the normal cone $-\mathcal{C}(\bm{B},\bm{a}_2) $. This implies that  $\mathcal{C}(\bm{A},\bm{a}_1) \cap -\mathcal{C}(\bm{B},\bm{a}_2) \neq \emptyset$. From the separating hyperplane theorem (see Appendix B), we can conclude that there is a supporting hyperplane that contains the contact patch and is also a separating hyperplane for the two objects.

Thus, when distance between bodies is zero, solutions $\bm{a}_1$ and $\bm{a}_2$ that satisfy Equations~\eqref{equation:re_contact_multiple_1} $\sim$~\eqref{equation:re_contact_multiple_5} also ensure that the bodies will be touching each other and not intersecting (although we enforced the contact constraints only at the ECPs). This proves our proposition.

\end{proof}
As stated previously, objects $F$ and $G$ are formed by union of convex bodies. Thus, there exists multiple pairs of bodies or contact patches that have potential contact. In this subsection, we use modified KKT conditions (Equations~\eqref{equation:re_contact_multiple_1} $\sim$~\eqref{equation:re_contact_multiple_5}) to model each contact patch between objects separately. By ensuring that for each (potential) contact patch, Equations~\eqref{equation:re_contact_multiple_1} $\sim$~\eqref{equation:re_contact_multiple_5} is satisfied implies that the two objects do not penetrate with each other.

\noindent
{\bf Summary of geometrically implicit time-stepping scheme}:
To summarize, our geometrically implicit time-stepper has four components: discretized Newton-Euler equations (Equation~\eqref{equ:discrete_NE}), kinematic map (Equation~\eqref{equ:kinematic}), friction models (Equation~\eqref{equation:friction} and~\eqref{equation:friction_c}) and contact constraints (Equations~\eqref{equation:re_contact_multiple_1} $\sim$~\eqref{equation:re_contact_multiple_5}) for patches which are in contact or may have potential contact. Thus, the system of  equations for each time-step is a mixed non-linear complementarity problem (MLCP) which is composed of equality constraints (Equations~\eqref{equ:discrete_NE},~\eqref{equation:friction},~\eqref{equation:re_contact_multiple_1} and~\eqref{equation:re_contact_multiple_2}) and complementarity constraints (Equations~\eqref{equation:friction_c},~\eqref{equation:re_contact_multiple_1} $\sim$~\eqref{equation:re_contact_multiple_5}). Thus, the equivalent contact points, 
associated contact impulses, and the configuration of the object  
are solved simultaneously. 

\section*{PLANAR SLIDING WITH PURE TRANSLATION}
In this section, we consider an object $F$ sliding with pure translation on the flat plane $G$, where the contact region contains multiple convex patches. We prove that in this setting, the state of the object at the end of the time step is uniquely determined although the ECPs and the contact impulses associated with the contact patches may not be unique. Without loss of generality, object $G$ is assumed to be fixed.
The axes of contact frame on each contact patch, $i$, are normal axis $\bm{n}_i\in R^3$, and tangential axes $\bm{t}_i   \in R^3 $, $\bm{o}_i \in R^3$. The pair of ECPs for the $i$th contact patch between $F$ and $G$ are $\bm{a}_{1_i}$ and $\bm{a}_{2_i}$. The vector from center of gravity of $F$ to ECP $\bm{a}_{2_i}$ is $\bm{r}_i = [a_{2x_i}- q_x, a_{2y_i} - q_y, a_{2z_i} - q_z]^T$ and wrenches are defined in Equation $\ref{equation:wrenches}$.
The state of the object F is $\bm{q} = [q_x, q_y, q_z,{^s\theta}_x, {^s\theta}_y, {^s\theta}_z]^T$. Let velocity $\bm{v} = [v_x , v_y , v_z]^T$ and spatial angular velocity be ${^s\bm{w}} = [{^sw}_x, {^sw}_y, {^sw}_z]^T$. Thus generalized velocity is $\bm{\nu} = [v_x , v_y , v_z, {^sw}_x, {^sw}_y, {^sw}_z]^T$.
The vector of external impulses and angular impulses are $
\bm{J}_{app} =  [J_x, J_y,-m\beta h+J_z]^T$ and
$\bm{L}_{app} = [L_{x\tau},L_{y\tau},L_{z\tau}]^T
$, where $\beta$ is the acceleration due to gravity. Thus, the generalized applied impulse is $\bm{P}_{app} =[J_x, J_y,-m\beta h+J_z,L_{x\tau},L_{y\tau},L_{z\tau}]^T $.

\noindent
{\bf Dynamic equations for pure translation}:
Since $G$ is a flat plane with zero curvature, we choose $\bm{n}$ as the normal axis of contact frame on each patch. Thus the  contact frame $(\bm{t}, \bm{o}, \bm{n})$ for each contact patch is same.
From equation~\ref{equ:discrete_NE}, the translational components of equations of motion can be written as:
\begin{equation}
 0=-m\bm{I}_{3}(\bm{v}^{u+1}-\bm{v}^{u}) +\bm{n}\sum_{i=1}^{n_c} p_{n_i}^{u+1}
 +\bm{t}\sum_{i=1}^{n_c} p_{t_i}^{u+1}+\bm{o}\sum_{i = 1}^{n_c} p_{o_i}^{u+1} +
\bm{J}^u_{app} 
\end{equation}

\noindent
Along the direction of contact frame ($\bm{n},\bm{t},\bm{o}$), 
\begin{align}
\label{equation:Euler_1}
m\bm{n}\cdot\bm{v}^{u+1} = \sum_{i=1}^{n_c} p_{n_i}^{u+1} +\bm{n}\cdot \bm{J}^u_{app} +m\bm{n}\cdot\bm{v}^{u}\\
\label{equation:Euler_2}
m\bm{t}\cdot\bm{v}^{u+1} = \sum_{i=1}^{n_c} p_{t_i}^{u+1} +\bm{t}\cdot \bm{J}^u_{app} +m\bm{t}\cdot\bm{v}^{u}\\
\label{equation:Euler_3}
m\bm{o}\cdot\bm{v}^{u+1} = \sum_{i=1}^{n_c} p_{o_i}^{u+1} +\bm{o}\cdot \bm{J}^u_{app} +m\bm{o}\cdot\bm{v}^{u}
\end{align}
\comment{
From equation~\ref{equ:discrete_NE}, the rotational dynamic equation is written as:
\begin{multline}
\label{equation:Euler_456}
0 = -{^s\bm{I}}_{cm}^{u+1} (^s\bm{w}^{u+1}-^s\bm{w}^{u})-^s\bm{w}^{u} \times( {^s\bm{I}}_{cm}^{u}{^s\bm{w}^{u}}) 
- \bm{n} \times\sum_{i=1}^{n_c} p_{n_i}^{u+1} \bm{r}^{u+1}_i\\
-\bm{t} \times \sum_{i=1}^{n_c} p_{t_i}^{u+1} \bm{r}^{u+1}_i
-\bm{o} \times\sum_{i=1}^{n_c} p_{o_i}^{u+1} \bm{r}^{u+1}_i
+ \bm{n}\sum_{i=1}^{n_c} p_{r_i}^{u+1} +\bm{L}^u_{app}
\end{multline}
}

\noindent
{\bf Normal velocity constraint}:
As proven in previous section, for each contact patch which keeps in contact, the associated pair of ECPs coincide with each other ($\bm{a}_{1_i}=\bm{a}_{2_i}$). Furthermore, as $F$ never loses contact with $G$, velocity of $F$ along normal direction should be zero, i.e., $\bm{n}\cdot \bm{v} = 0$.
\comment{
\begin{figure}
\centering
\includegraphics[width=4in]{figure_surface_contact}
\caption{The surface contact of the object A and the flat plane}
\label{figure:3D_surface} 
\end{figure} 
}

\noindent
{\bf Friction model for pure translation}:
For pure translation, angular velocity of the object is zero and velocity of any point of the object stays the same.  Consider the friction model (Equations~\eqref{equation:friction} and~\eqref{equation:friction_c} ) for contact patch $i$. Without loss of generality, we assume that $e_{t_i}, e_{o_i}, e_{r_i}$ and coefficient of friction $\mu_i$ is same for each patch. Thus, slip velocity $\sigma_i$ on each patch $i$ has same value, which is $\sigma_i = \sqrt{ (e_t\bm{t}\cdot\bm{v}^{u+1})^2 +(e_o\bm{o}\cdot \bm{v}^{u+1})^2}$ . Therefore, adding the friction constraints for each contact, we obtain
\begin{align}
\label{equation:friction_modified_1}
0&=
e^{2}_{t}\mu \bm{t}\cdot\bm{v}^{u+1}\sum_{i=1}^{n_c}p_{n_i}^{u+1} +
\sigma\sum_{i=1}^{n_c}p_{t_i}^{u+1}\\
\label{equation:friction_modified_2}
0&=
e^{2}_{o}\mu \bm{o}\cdot\bm{v}^{u+1}\sum_{i=1}^{n_c}p_{n_i}^{u+1} +
\sigma\sum_{i=1}^{n_c}p_{o_i}^{u+1}\\
\label{equation:friction_modified_3}
0&= \sigma\sum_{i=1}^{n_c}p_{r_i}^{u+1}\\
\label{equation:friction_modified_4}
\sigma &= \sqrt{ (e_t\bm{t}\cdot\bm{v}^{u+1})^2 +(e_o\bm{o}\cdot \bm{v}^{u+1})^2}
\end{align}
{\bf Analytical solution for sum of contact impulses and linear velocity of the object}:
Assuming isotropic friction ($e_t = e_o$), we now combine translational dynamic Equations~\eqref{equation:Euler_1}$\sim$~\eqref{equation:Euler_3}, normal velocity constraint ($\bm{n}\cdot \bm{v} = 0$), and friction model (Equation~\eqref{equation:friction_modified_1}$\sim$~\eqref{equation:friction_modified_4}) to derive the closed form solution for the sum of contact impulses and linear velocity of the object. 

\begin{proposition}{Equations~\eqref{equation:Euler_1}$\sim$~\eqref{equation:friction_modified_4} together model the motion of planar sliding with pure translation. Furthermore, by assuming isotropic friction  ($e_t = e_o$), there exists analytical solutions for the sum of contact impulses ($\sum_{i=1}^{n_c} p_{t_i}^{u+1}$, $\sum_{i=1}^{n_c} p_{o_i}^{u+1}$ and $\sum_{i=1}^N p_{r_i}^{u+1}$) and linear velocity of the object ($\bm{v}^{u+1}$):}
\begin{align}
\label{equation:impulse_t}
\sum_{i=1}^{n_c} p_{t_i}^{u+1} &= -\frac{e_t\mu\bm{n}\cdot\bm{J}^u_{app}(\bm{t}\cdot \bm{J}^u_{app} +m \bm{t}\cdot \bm{v}^u) }{\sqrt{(m \bm{t}\cdot \bm{v}^u +\bm{t}\cdot \bm{J}^u_{app})^2+(m \bm{o}\cdot \bm{v}^u +\bm{o}\cdot \bm{J}^u_{app})^2}} \\
\label{equation:impulse_o}
\sum_{i=1}^{n_c} p_{o_i}^{u+1} &= -\frac{e_o\mu\bm{n}\cdot\bm{J}^u_{app} (\bm{o}\cdot \bm{J}^u_{app} +m \bm{o}\cdot \bm{v}^u)}{\sqrt{(m \bm{t}\cdot \bm{v}^u +\bm{t}\cdot \bm{J}^u_{app})^2+(m \bm{o}\cdot \bm{v}^u +\bm{o}\cdot \bm{J}^u_{app})^2}}\\
\label{equation:impulse_r}
\sum_{i=1}^N p_{r_i}^{u+1}&= 0 \\
\label{equation:impulse_n}
\sum_{i=1}^{n_c} p_{n_i}^{u+1}&= \bm{n}\cdot\bm{J}^u_{app}
\end{align}
where Equations~\eqref{equation:impulse_t}$\sim$\eqref{equation:impulse_n} are the analytical solutions for sum of impulses. The analytical solution of the linear velocity can be derived by substituting Equations~\eqref{equation:impulse_t},~\eqref{equation:impulse_o} and~\eqref{equation:impulse_n} into Equations~\eqref{equation:Euler_1},~\eqref{equation:Euler_2} and~\eqref{equation:Euler_3}.  
\end{proposition}
\begin{proof}
Substituting $\bm{n}\cdot \bm{v} = 0$ into Equation~\eqref{equation:Euler_1}, we prove that $\sum_{i=1}^{n_c} p_{n_i}^{u+1} = \bm{n}\cdot\bm{J}^u_{app}$. Them we substitute Equation~\eqref{equation:Euler_2} and~\eqref{equation:Euler_3} into Equations~\eqref{equation:friction_modified_1} and~\eqref{equation:friction_modified_2}. After simplification, we get the closed form expression for $\sum_{i=1}^{n_c} p_{t_i}^{u+1}$ and 
$\sum_{i=1}^{n_c} p_{o_i}^{u+1}$. From Equation~\eqref{equation:friction_modified_3}, because $\sigma \neq 0$, thus $\sum_{i=1}^{n_c} p_{r_i}^{u+1}= 0$.
\end{proof}

The analytical solution presented in this section solves for the state of the object directly, while our general method requires to solve the geometrically implicit model numerically. Apart from being useful in the special case of pure translation, as we will show in the next section, the analytical solution is useful to validate our numerical results. 
\comment{
Because object G is a flat plane, so all the ECPs $\bm{a}_{2_i}$ (ECPs belong to object G) should lie on the plane. We have additional constraints for $\bm{a}_{2_i}$:
\begin{equation}
\bm{n} \cdot \bm{a}_{2_i} = d, \ \forall i = 1,...,n_c
\end{equation}
where $\bm{n}$ is the normal vector of the plane and $d$ is the distance of the plane from the origin of the world frame. 
}


\begin{figure*}%
\centering
\begin{subfigure}{0.7\columnwidth}
\includegraphics[width=\columnwidth]{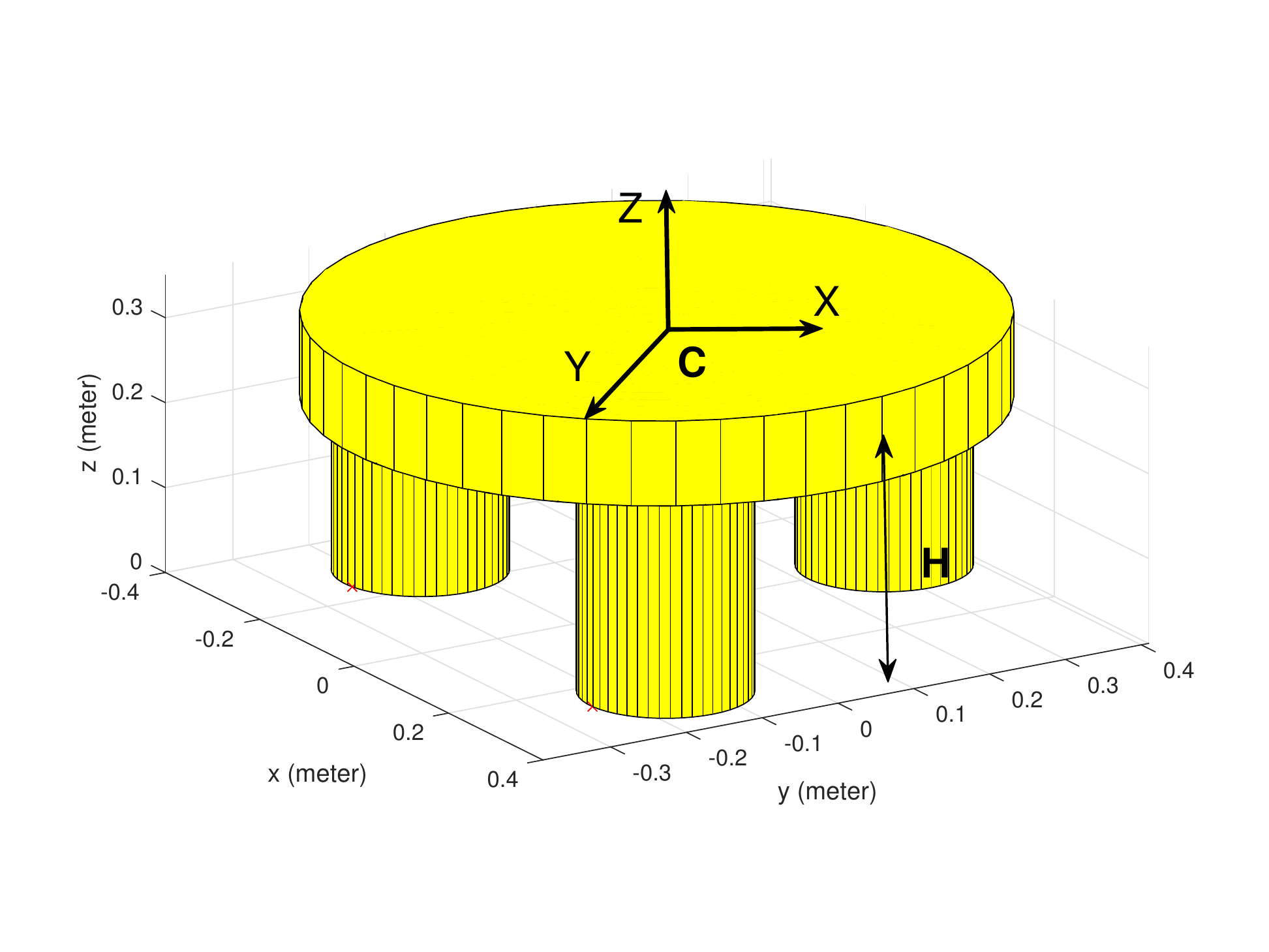}%
\caption{Object with three contact patches on the plane. }
\label{figure:ex1_picture} 
\end{subfigure}\hfill%
\begin{subfigure}{0.65\columnwidth}
\includegraphics[width=\columnwidth]{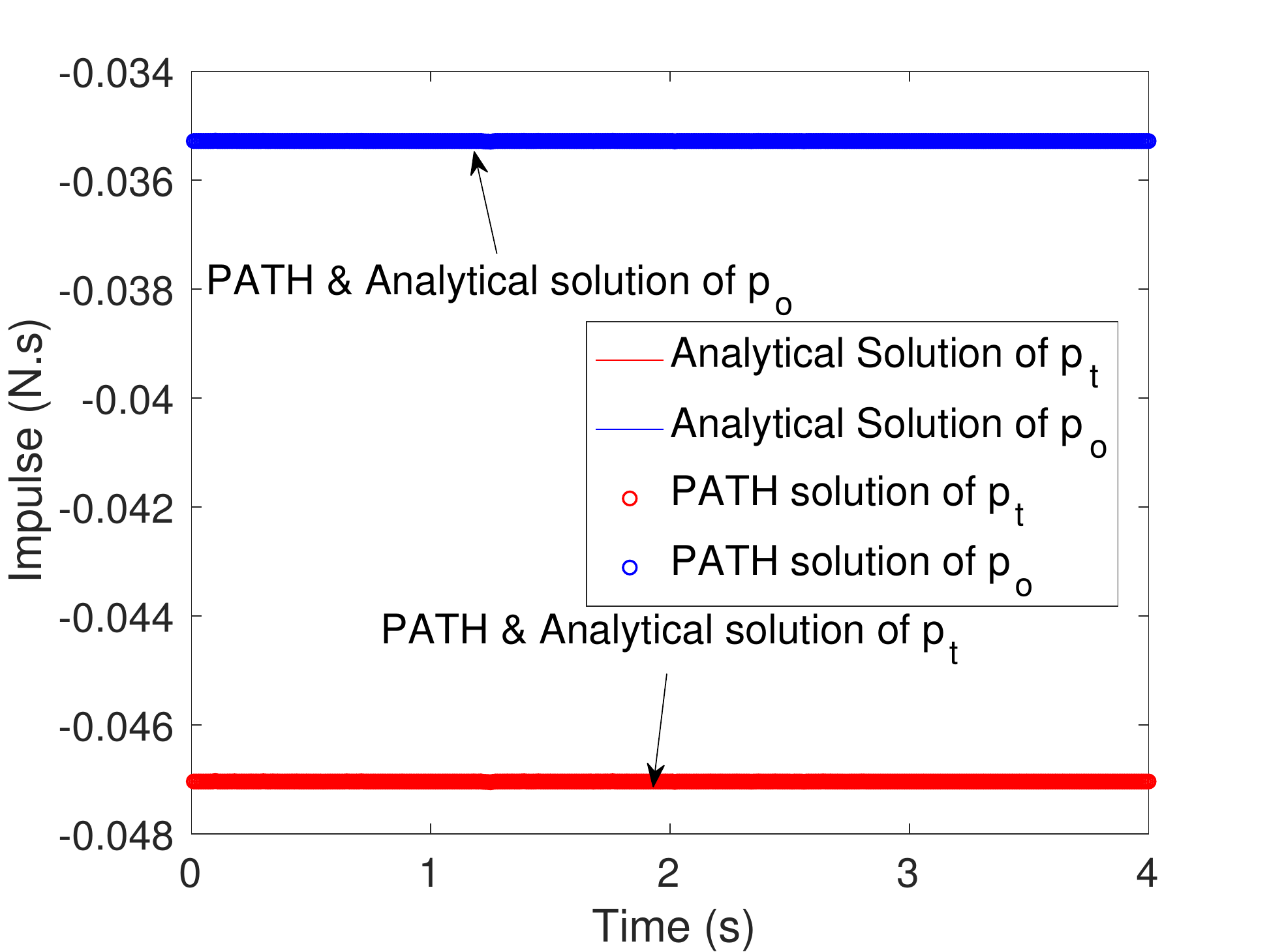}%
\caption{Analytically computed and numerically computed  contact impulses match (within numerical tolerance of $10^{-6}$). }
\label{figure:ex1_impulse} 
\end{subfigure}\hfill%
\begin{subfigure}{0.65\columnwidth}
\includegraphics[width=\columnwidth]{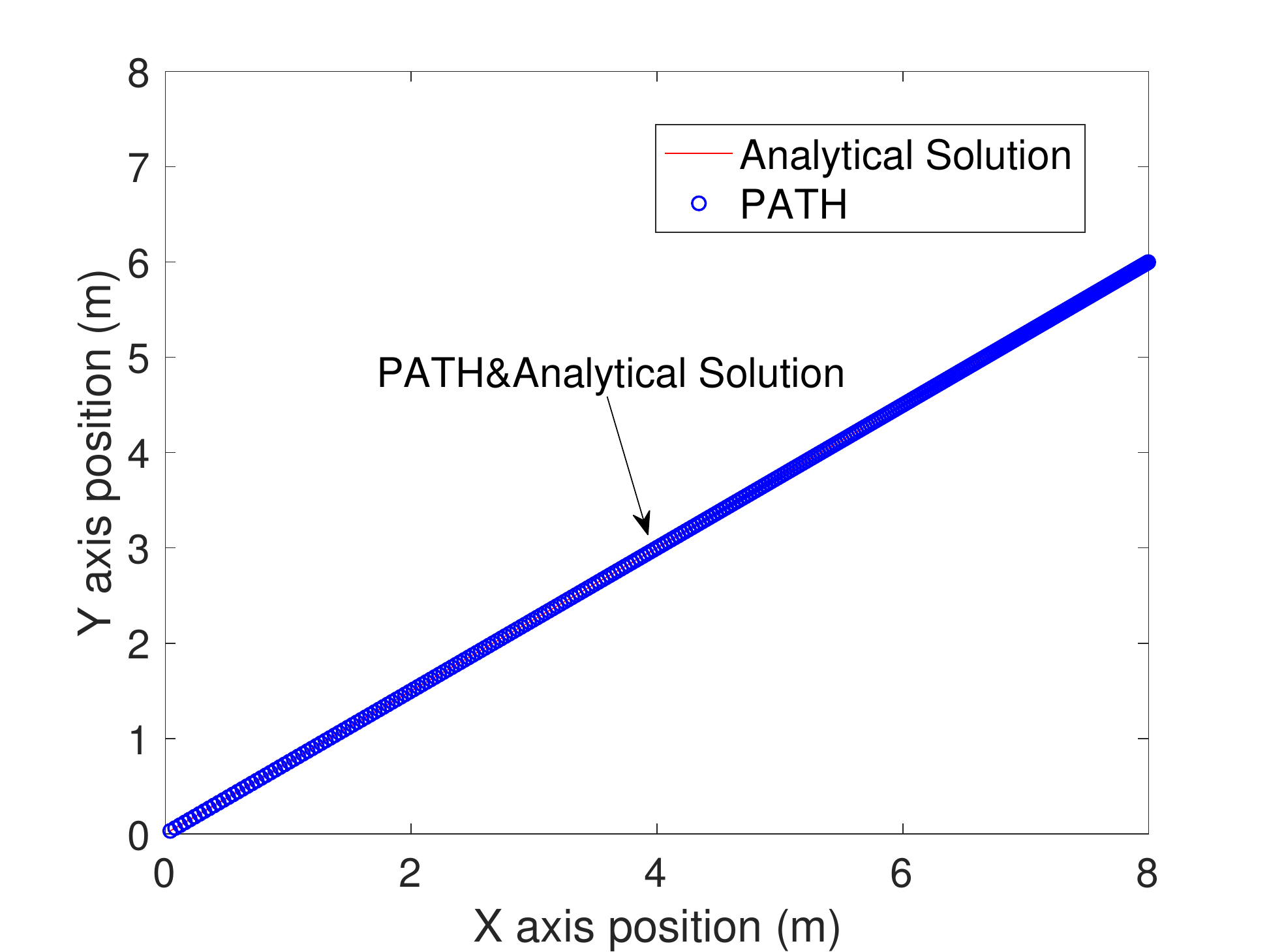}%
\caption{Analytically computed and numerically computed  trajectories match (within numerical tolerance of $10^{-6}$).}
\label{figure:ex1_delta} 
\end{subfigure}%
\caption{For object with three contact patches undergoing pure translation our numerical solution matches the analytical solution.}
\label{Example1}
\end{figure*}
\begin{figure*}%
\centering
\begin{subfigure}{0.65\columnwidth}
\includegraphics[width=\columnwidth]{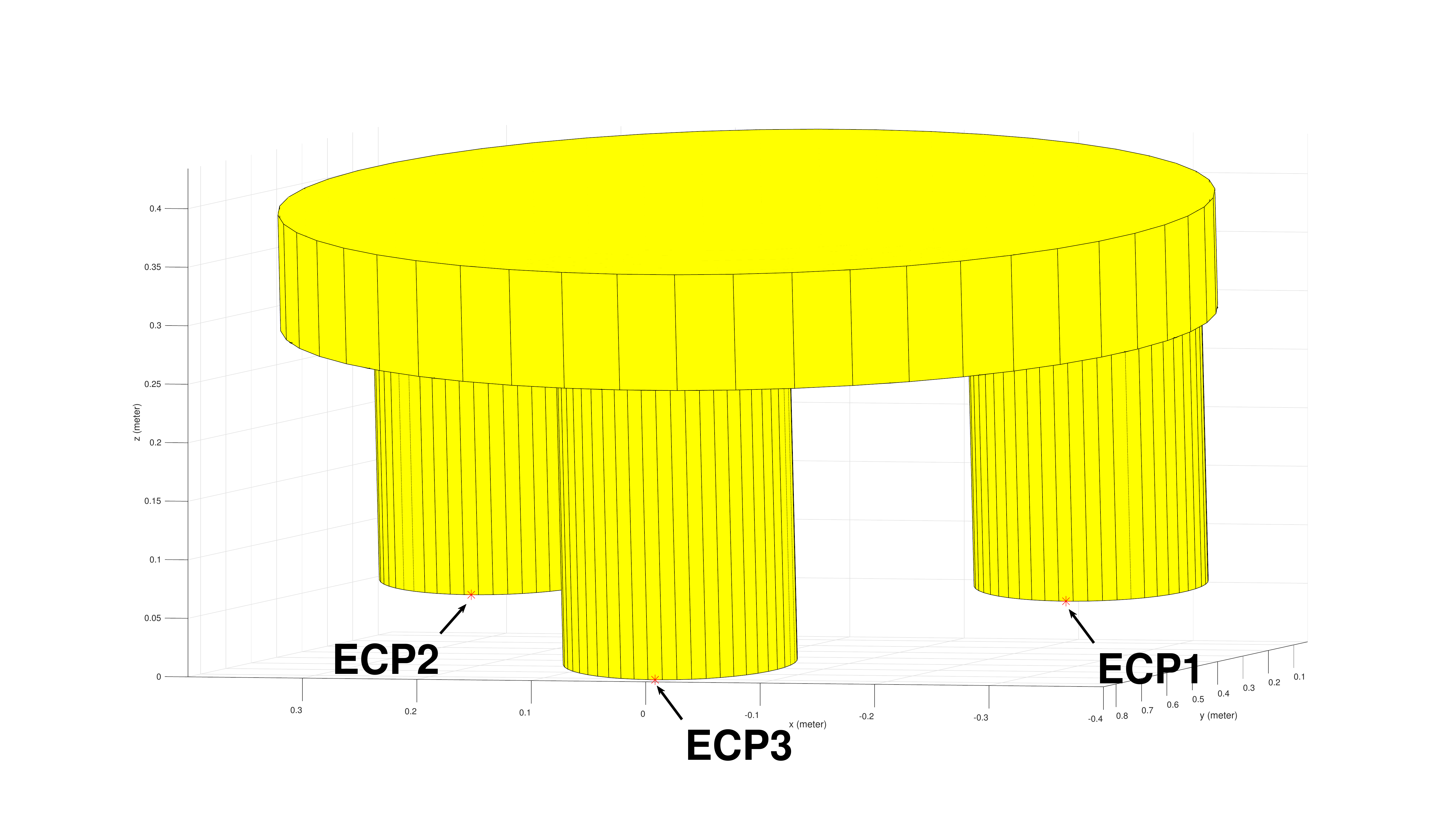}%
\caption{Snapshot of the object toppling on the plane. The red dots mark the three contact points. }
\label{figure:ex2_picture} 
\end{subfigure}\hfill%
\begin{subfigure}{0.65\columnwidth}
\includegraphics[width=\columnwidth]{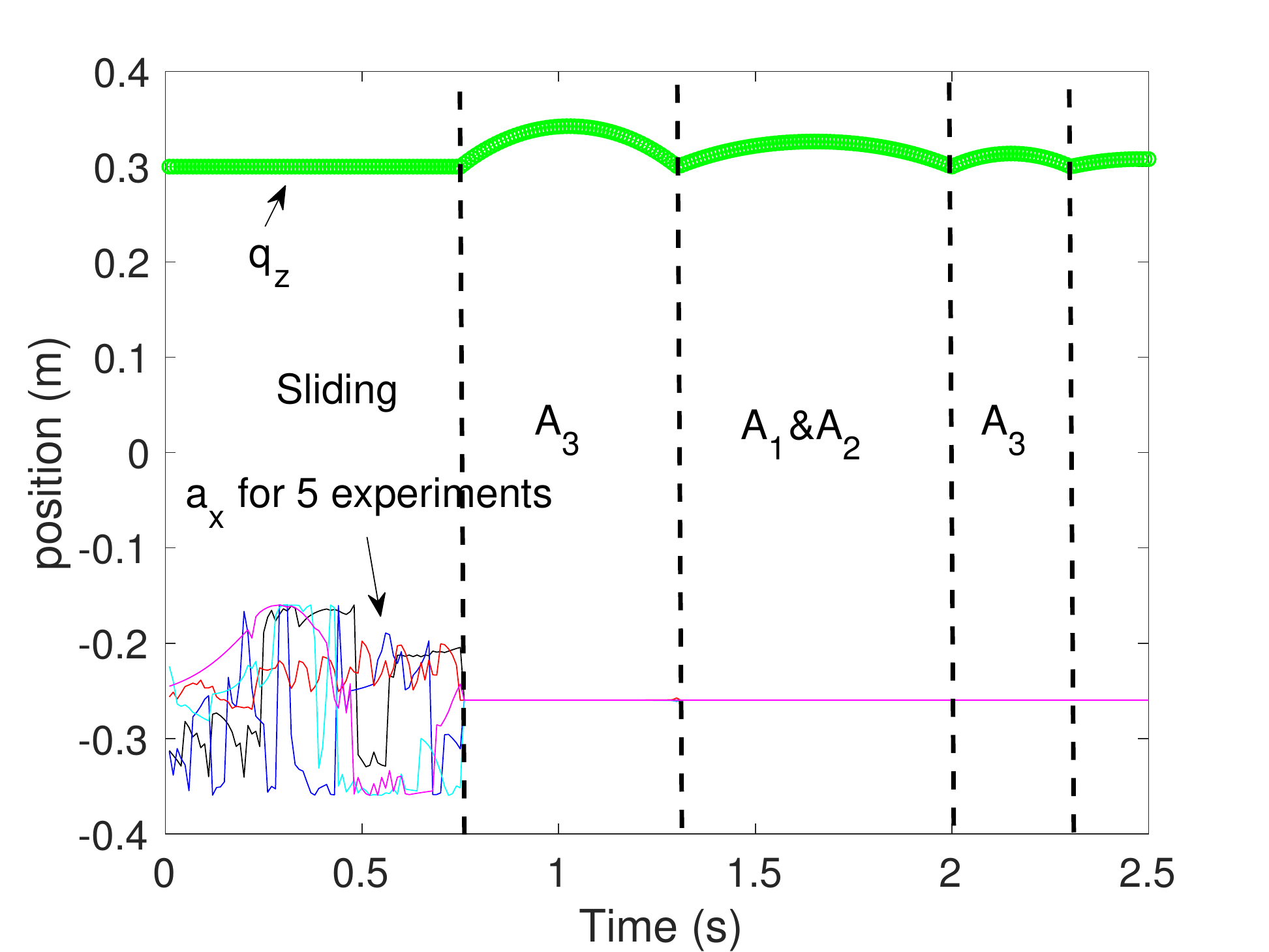}%
\caption{$X$-coordinate of contact point $1$ for $5$ different simulations.}
\label{figure:ex2_ax} 
\end{subfigure}%
\begin{subfigure}{0.65\columnwidth}
\includegraphics[width=\columnwidth]{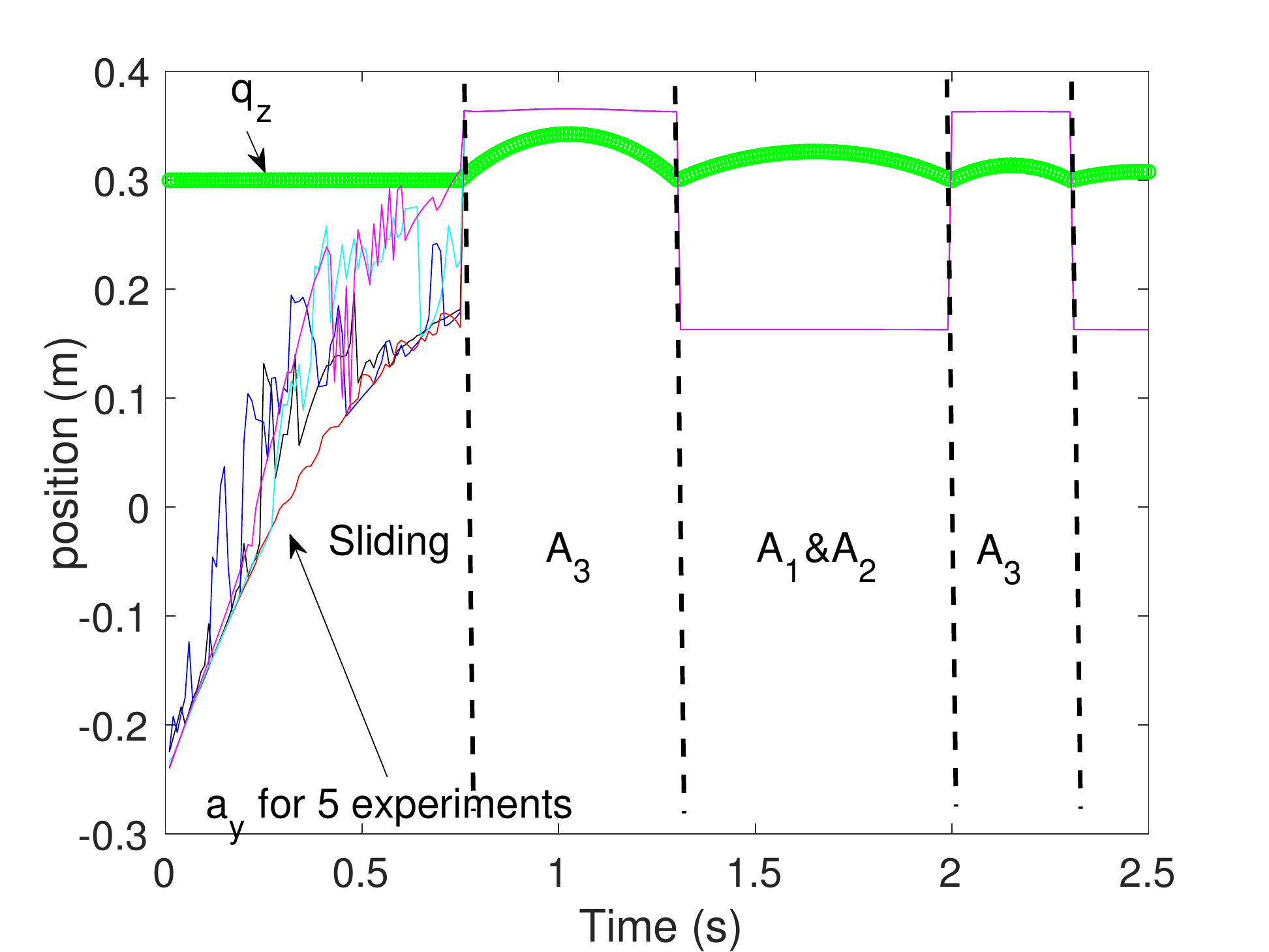}%
\caption{$Y$-Coordinate of contact point $1$ for $5$ different simulations.}
\label{figure:ex2_ay} 
\end{subfigure}%
\caption{Simulation of a three-legged object on a plane, where the motion transitions between sliding and toppling. Our simulation captures the contact transitions between patch (or surface) contacts during sliding to point contacts during toppling. Although the (equivalent) contact point may be non-unique during sliding, the position of the body is unique, and when toppling starts, i.e., contact transitions to point contact, the contact point is unique.}
\label{Example2}
\end{figure*}
\comment{
\begin{figure*}%
\centering
\begin{subfigure}{0.45\columnwidth}
\includegraphics[width=\columnwidth]{figure_2_patch}%
\caption{Object with two contact patches on the plane. }
\label{figure:ex3_picture} 
\end{subfigure}\hfill%
\begin{subfigure}{0.65\columnwidth}
\includegraphics[width=\columnwidth]{figure_2_patch_position}%
\caption{Trajectories of position of object from repeated experiments .}
\label{figure:ex3_position} 
\end{subfigure}%
\begin{subfigure}{0.65\columnwidth}
\includegraphics[width=\columnwidth]{figure_2_patch_po1}%
\caption{Contact impulse $p_o$ of ECP in contact patch 1 from repeated experiments .}
\label{figure:ex3_po1} 
\end{subfigure}%
\caption{Object with two contact patches sliding and rotating about normal axis on the plane.}
\label{Example3}
\end{figure*}
}
\section*{NUMERICAL RESULTS}
We have tested our methodology for objects moving on the plane where the contact between the object and the plane can be modeled as a union of convex patches. We now present representative numerical simulations to illustrate key aspects of our methodology. For both examples, we use the complementarity solver, PATH~\cite{StevenP.Dirkse1995}, to solve the NCP at each time step. In our first simulation, we consider a table with three legs translating on a flat plane. For this situation, we use the analytical solution as our ground truth (Equations~\eqref{equation:impulse_t} $\sim$~\eqref{equation:impulse_n}) and compare it with the numerical results obtained from solving our NCP problem model (that was formulated without making the special assumption that the motion is pure translation). This example serves as a sanity check to validate that our method gives the solution that should be obtained. In our second simulation, we consider an object sliding on the plane along $y$ axis, on which, an angular impulse was applied about the $x$ axis. Through this example, we illustrate that although there does not exists unique solution for ECPs and contact impulses on each contact patch, the state of the object is still unique. To show this, we repeat this experiment five times by changing the initial guess for the PATH solver. We also show that our method allows seamless transition between different contact modes (e.g., surface contact to point contact).

\noindent
{\bf Example 1: Translating object on a plane}:
Figure $\ref{figure:ex1_picture}$ shows an object with three legs translating on a plane with no rotation. The contact area between the object and the plane is an union of three disk contact patches. The body frame is fixed at the center of gravity $C$ of the object. The height of center of gravity $H = 0.3m$. Each leg is a cylinder with radius $0.1 m$. The distance from $C$ to the axis of each cylinder is $0.3 m$. 
We use a fixed time step, $h = 0.01 s$. 
The mass of the object is $m = 5 kg$. Let acceleration due to gravity be $\beta = 9.8 m/s^2 $, constants of friction ellipsoid be $e_t = 1, e_o = 1, e_r=1$ and coefficient of friction be $\mu = 0.12$. The initial configuration $q = [0, 0, 0.3, 0, 0, 0, 0]^T$ and initial generalized velocity $\bm{\nu} = [4, 3, 0, 0, 0, 0]^T$. 

Figure~\ref{figure:ex1_impulse} show the numerical and analytical solutions for the sum of contact impulses in tangential directions $\sum_{c_i}^{n_c}p_{t_i}$ (the red line represents analytical solution, and the red circle marker represents numerical solution) and $\sum_{c_i}^{n_c}p_{o_i}$ ( the blue line represents analytical solution, and the blue circle marker represents numerical solution). Here we use analytical solution as a ground truth. In Figure~\ref{figure:ex1_impulse}, because our numerical solution matches analytical solution within tolerance of $10^{-6}$, the red and blue markers overlap the red and blue lines.
Figure $\ref{figure:ex1_delta}$ shows the numerical and analytical solution for the trajectory of object along $x$ and $y$ axis (blue marker for numerical solution and red line for analytical solution). Our numerical solution matches the analytical solution (therefore, we do not see two separate lines on the plot). 

\noindent
{\bf Example 2: Translating and rotating object on a  plane}:
We consider a table (with same geometry and inertial parameters as in Example $1$) with three legs undergoing general three-dimensional rotational and translational motion on a plane object. The table is initially sliding along the $y$ axis. At $t \approx 0.75$ seconds (see Figure~\ref{figure:ex2_ax} and~\ref{figure:ex2_ay}), we provide a angular impulse $L_{x\tau} = -5 N.m$ on the object. The value of the impulse is chosen so that the table tilts about the $x$-axis  but does not topple over. 

This example is used to illustrate that our method can allow objects to automatically transition between different contact modes (surface, point, and line and also making and breaking of contact). In addition, it empirically validates that the non-uniqueness of ECPs and associated contact impulses on each contact patch do not affect the uniqueness of state of the object. Furthermore, when there is a transition from patch contact to point contact, where the ECP is unique, our method obtains the unique ECP, irrespective of the non-unique ECP that was obtained when there was surface contact.  

The initial velocity and configuration of the object are $\bm{\nu} = [0, 1, 0, 0, 0, 0]^T$ and $\bm{q} = [0, 0, 0.3, 1, 0, 0, 0]^T$. We repeat the simulation for $5$ times based on different initial values for solving the NCP for the first time step. As Figure $\ref{figure:ex2_picture}$ shows, after the application of the impulse, the object tilts on the plane. Only one leg has point contact with the plane (denoted by ECP3) while the other two legs just lose contact with the plane. 
Figures~\ref{figure:ex2_ax} and~\ref{figure:ex2_ay} shows the trajectory of ECP1 during the motion. $A_1,A_2, A_3$ represents the index of contact patch where ECP1, ECP2, and ECP3 belongs, respectively. In Figures~\ref{figure:ex2_ax} and~\ref{figure:ex2_ay}, the region from the start to the first vertical dashed line, labeled 'Sliding' is the region where all the three patches are in contact and the object is sliding. The region labeled $A_3$ within two vertical dashed lines indicates that the object is tilted and there is only point contact between $A3$ and the plane (there is no contact between the other legs and the plane). In the region labeled $A_1 \& A_2$, there is point contact between each of $A_1$ and $A_2$ and the plane. Note that during sliding, the trajectory of ECP1 is not unique and depends on the initial guess for solving the MNCP for the first time step (the $5$ lines of different colors correspond to $5$ different runs). However, after $t \approx 0.75$ seconds, when the object starts rotating about the $x$ axis and there is transition from surface contact on all three patches to single point contact with $A3$. Thus ECP1 changes to the closest point on the patch $1$ from the plane. As shown in Figures~\ref{figure:ex2_ax} and~\ref{figure:ex2_ay}, $X$ coordinate of ECP1 stays constant (although the point on the surface that is ECP1 changes as the body tilts), since the rotation is about $x$ axis, but there is a  jump in $Y$ coordinate. When patch $1$ and $2$ have point contact with the plane, ECP1 changes to the contact point. For point contact, the ECPs are unique and we see that from our simulation that irrespective of the starting condition, when the transition to point contact occurs we get the same solution (the different colored lines coincide in the regions after the 'Sliding' region). 
The green line shows the $z$ coordinate of the center of the object. This remains same (as also the $x$ and $y$ coordinates, which are not shown due to lack of space) in all the experiments. This demonstrates that although the ECP may not be unique (hence, contact impulses may not be unique), the net contact impulse and the net motion is always unique (again, the plot for the contact impulses is not shown due to space constraints). Furthermore, the oscillation of $z$ coordinate shows that the object is tilting back and forth after the application of the impulsive moment about the $x$ axis.

\comment{
\subsection*{Example 3: The object with two contact patches sliding with rotating about the normal axis on the Plane}
Shown in figure~\ref{figure:ex3_picture},  we modified the object in previous example to be with two contact patches. The initial velocity $\bm{\nu} = [4, 3, 0, 0, 0, 1]$ and initial configuration $q = [0, 0, 0.3, 0, 0, 0, 0]$. 

Figure~\ref{figure:ex3_position} plots the trajectories of the position of object from five repeated experiments. The result illustrates the uniqueness of the state of the object.  Figure~\ref{figure:ex3_po1} plots the contact impulse $p_o$ on the ECP of contact patch $1$, which shows the non-uniqueness of the contact impulses.
}
\section*{CONCLUSION}
In this paper, we presented a geometrically implicit time-stepping method for solving dynamic simulation problems with multiple convex contact patches. We combine the collision detection with numerical integration, which allows us to solve for an equivalent contact point (ECP) on each contact patch as well as the contact wrenches simultaneously. We prove that although we model each contact patch with an ECP, the non-penetration constraints at the end of the time-step are always satisfied. Our numerical simulation results demonstrate that although the ECP and its associated contact wrenches on each contact patch may not be unique, the state (configuration and velocity) of the object is still unique. We present numerical results illustrating that our method can automatically transition among different contact modes (non-convex contact patch, point and line). For pure translation, we can solve for the state of the object in closed form as well as the constraints for ECPs and contact wrenches such that the surface contact will be maintained at the end of time step. 

In this paper, we have given empirical evidence that the proposed method generates a unique solution for the state of the object, although the contact impulses and ECP generated by the geometrically implicit method is non-unique. In future work, we want to obtain a theoretical proof of the claim. Furthermore, we want to exploit the use of this motion prediction algorithm with union of convex contact patches for manipulation planning.



\bibliographystyle{asmems4}


%
% ...



\appendix       

\section*{Appendix A: Normal Cone}
As shown in Figure~\ref{figure:contact_nonpenetration}, when a contact patch is described by intersection of convex functions, there can be contact points lying at the intersection of multiple functions (vertices and points on edges). The normals at these points are not uniquely defined. For any point $\bm{x}$ that lies at intersection of multiple functions, say $f_i(\bm{x}) = 0$, where $i$ belongs to an index set $II$, we can define a normal cone, $\mathcal{C}(\bm{F},\bm{x})$, that consists of all vectors in the conic hull of the normals for each function of object $F$ at $\bm{x}$ as:
$\mathcal{C}(\bm{F},\bm{x}) = \{ \bm{y} \vert  \bm{y} = \sum_{i \in II} l_i \nabla f_i(\bm{x}), \quad l_i\ge 0\}$,
where $l_i$ are non-negative constants.
Note that when one of the contact patches have a unique normal defined for all points on the patch, this normal can be used as a common normal even if for the other object the contact normal is not uniquely defined. When the contact normal is not unique, we can choose the common normal as any vector that lies in the intersection of normal cone on one object and the negative of the normal cone on the other object. The normal cone for a line or surface also defines the set of supporting hyperplanes to the line or surface~\cite{Rockafellar1997}.
\section*{Appendix B: Separating Hyperplane Theorem}
The separating hyperplane theorem states that two nonempty convex sets in $\mathbb{R}^n$ can be properly separated by a hyperplane if and only if their interiors are disjoint~\cite{Rockafellar1997}. When objects are separate, the normal to the separating hyperplane is along the line joining the closest points on the two sets. When two sets have line or surface contact without intersection, a separating hyperplane is also a supporting hyperplane for the contact line or surface on both the sets. Thus, in this case, the separating hyperplane theorem implies: two non-empty convex objects can have a common supporting hyperplane for the contact line or surface on both sets if and only if their interiors are disjoint. Thus,

\begin{enumerate}
\item When the distance between two objects $F$ and $G$ is greater than zero and ECPs $\bm{a}_1$ and $\bm{a}_2$ on $F$ and $G$ are the closest points on the boundary of two convex objects, the vector $\bm{a}_2 - \bm{a}_1$ lies within the set $\mathcal{C}(\bm{F},\bm{a}_1)$ and also within the set $-\mathcal{C}(\bm{G},\bm{a}_2)$, thus $\mathcal{C}(\bm{F},\bm{a}_1) \cap -\mathcal{C}(\bm{G},\bm{a}_2) \neq \emptyset$.

\item When two objects' distance is zero and they have line or surface contact without intersection, ECP $\bm{a}_1$ and $\bm{a}_2$ lie on the contact patch. The intersection of set $\mathcal{C}(\bm{F},\bm{a}_1)$ and set $-\mathcal{C}(\bm{G},\bm{a}_2)$ defines the set of supporting hyperplanes for contact patch on $F$ and $G$. Thus $\mathcal{C}(\bm{F},\bm{a}_1) \cap -\mathcal{C}(\bm{G},\bm{a}_2) \neq \emptyset$. Notice that $\bm{a}_1$ and $\bm{a}_2$ can be any point lies on the contact patch and they do not need to coincide with each other.

\item If the distance is zero and the two objects penetrate each other, ECP $\bm{a}_1$ and $\bm{a}_2$ lie on the contact patch. Then $\mathcal{C}(\bm{F},\bm{a}_1) \cap -\mathcal{C}(\bm{G},\bm{a}_2) = \emptyset$, which implies that there is no hyperplane that can separate the two objects.
\end{enumerate}


\end{document}